\documentclass[11pt]{article}

\usepackage[preprint]{acl}

\usepackage{times}
\usepackage{amsmath}
\usepackage{latexsym}
\usepackage{amsfonts}
\usepackage{booktabs}
\usepackage{algorithm}
\usepackage{tabularray}
\usepackage{algorithmic}
\usepackage{multirow}
\usepackage[normalem]{ulem}
\usepackage{booktabs}
\usepackage{xspace}
\usepackage{rotating}
\usepackage[T1]{fontenc}

\usepackage[utf8]{inputenc}

\usepackage{microtype}

\usepackage{inconsolata}

\usepackage{graphicx}
\usepackage{xcolor}
\usepackage[most]{tcolorbox}
\definecolor{colorPrompt}{RGB}{240, 242, 245}
\definecolor{colorOurModel}{RGB}{235, 248, 242}
\definecolor{colorGRPO}{RGB}{255, 248, 235}
\definecolor{colorBase}{RGB}{252, 235, 235}
\definecolor{frameColor}{RGB}{60, 60, 60}

\usepackage{amssymb}
\usepackage{amsthm}
\newtheorem{theorem}{Theorem}[section]
\newtheorem{proposition}[theorem]{Proposition}

\newcommand{\ourmethod}{SPARD\xspace}

\title{\ourmethod: Self-Paced Curriculum for RL Alignment via Integrating \\ Reward Dynamics and Data Utility}

\author{
    \textbf{Xuyang Zhi\textsuperscript{1}},
    \textbf{Peilun Zhou\textsuperscript{2}}, 
    \textbf{Chengqiang Lu\textsuperscript{2}},
    \textbf{Hang Lv\textsuperscript{1}},
    \textbf{Yiwei Liang\textsuperscript{1}},
    \\
    \textbf{Rongyang Zhang\textsuperscript{1}},
    \textbf{Yan Gao\textsuperscript{2}},
    \textbf{Yi Wu\textsuperscript{2}},
    \textbf{Yao Hu\textsuperscript{2}},
    \textbf{Hongchao Gu\textsuperscript{1}},
    \\
    \textbf{Hao Wang\textsuperscript{1}},
    \textbf{Defu Lian\textsuperscript{1}},
    \textbf{Enhong Chen\textsuperscript{1}}
    \\
    \textsuperscript{1}University of Science and Technology of China, \textsuperscript{2}Xiaohongshu Inc.
    \\
}

\usepackage[most]{tcolorbox}

\tcbset{
    commonpromptstyle/.style={
        colback=white,
        colframe=gray!50,
        colbacktitle=gray!15,
        coltitle=black,
        fonttitle=\bfseries,
        boxrule=0.5pt,
        arc=2pt,
        left=5pt, right=5pt, top=5pt, bottom=5pt,
        toptitle=2pt, bottomtitle=2pt,
        fontupper=\small 
    }
}

\newtcolorbox[auto counter]{promptbox}[3][]{
    commonpromptstyle,
    breakable,
    title={Prompt~\thetcbcounter: #2},
    label={#3}, 
    #1
}

\newtcolorbox[use counter from=promptbox]{promptbox*}[2][]{
    commonpromptstyle,
    width=\textwidth, 
    float*=t,         
    title={Prompt~\thetcbcounter: #2},
    #1
}

\begin{document}
\maketitle

\begin{abstract}

The evolution of Large Language Models (LLMs) is shifting the focus from single, verifiable tasks toward complex, open-ended real-world scenarios, imposing significant challenges on the post-training phase. In these settings, the scale and complexity of reward systems have grown significantly, transitioning toward multi-objective formulations that encompass a comprehensive spectrum of model capabilities and application contexts. However, traditional methods typically rely on fixed reward weights, ignoring non-stationary learning dynamics and struggling with data heterogeneity across dimensions.
To address these issues, we propose  \ourmethod, a framework that establishes an automated, self-paced curriculum by perceiving learning progress to dynamically adjust multi-objective reward weights and data importance, thereby synchronizing learning intent with data utility for optimal performance. Extensive experiments across multiple benchmarks demonstrate that  \ourmethod significantly enhances model capabilities across all domains.


\end{abstract}

\section{Introduction}

Recently, Large Language Models (LLMs) have seamlessly integrated into people's daily lives and professional workflows. 
As application scenarios become increasingly diverse and complex, the capability evolution of LLMs is accelerating from single verifiable tasks such as mathematical reasoning and code generation  \cite{deepseekr1,Tulu3,simplerl} toward open-ended real-world scenes like general dialogue and deepresearch \cite{drtulu,rlmt, huang2024chemevalcomprehensivemultilevelchemical,zhang2025ragigbenchinnovativeevaluationragbased,liang2025adaptiveschemaawareeventextraction,yin2025featureinteractionfeaturegeneration}.
This paradigm shift imposes significantly higher demands on the post-training phase, not only requiring the model to uphold objective factual accuracy but also demanding it to cater to subjective perceptual preferences.

\begin{figure}[t]
    \centering
    \includegraphics[width=1.00\columnwidth]{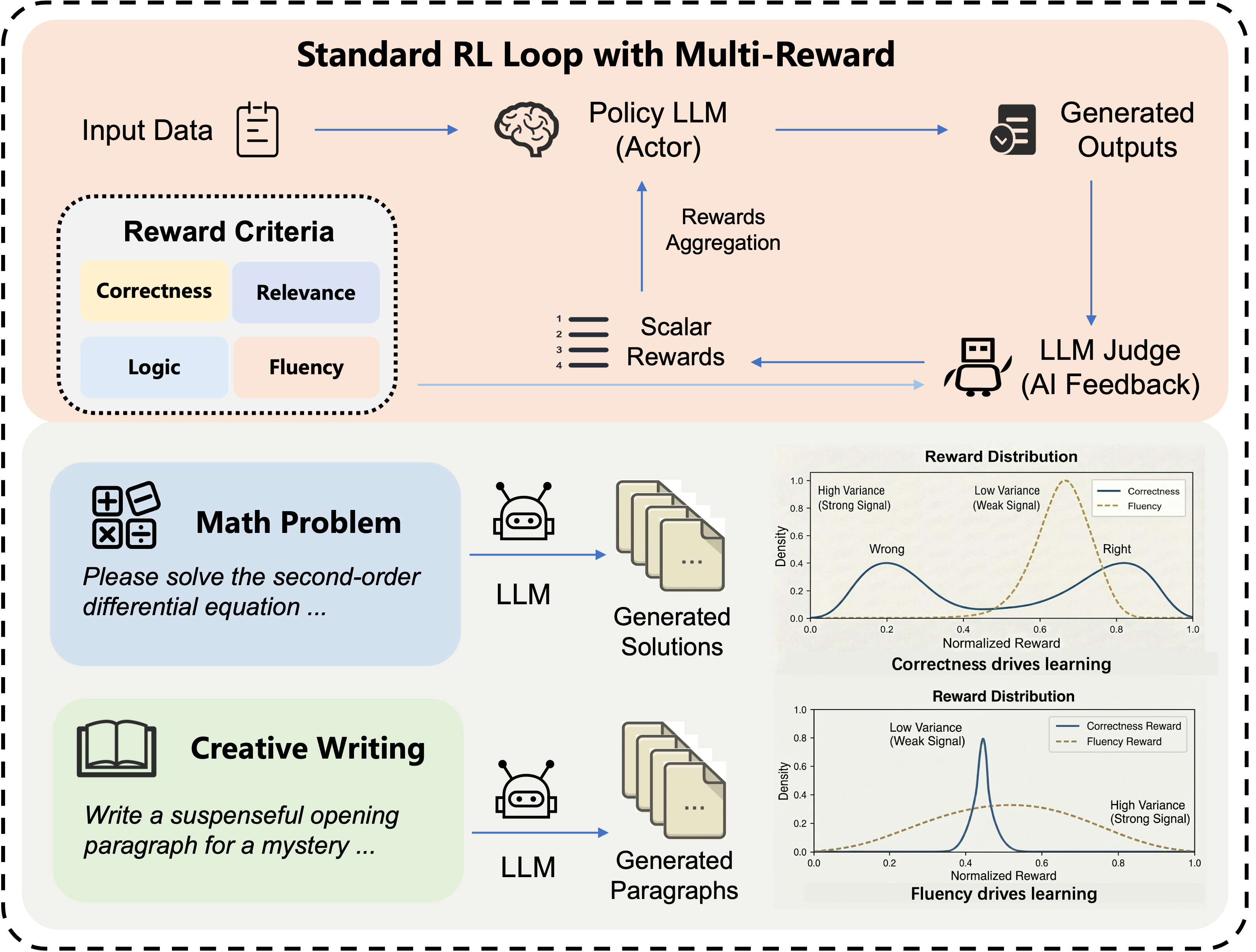}
    \caption{Illustration of the standard \textbf{Multi-Reward} RL loop and examples of training characteristics across diverse data types. The upper panel depicts the workflow of generating multi-reward via an LLM judge and aggregating them for policy updates. The lower panel highlights \textbf{data heterogeneity}, demonstrating that different types of input data differentially impact specific reward dimensions during training.}
    \label{fig:motivation}
    \vspace{-0.3cm}
\end{figure}


In these scenarios, the definition of rewards has evolved into multi-objective frameworks covering diverse criteria like correctness and fluency \cite{rar,rubric-anchors}, as shown in Figure ~\ref{fig:motivation}. However, effectively leveraging these multi-dimensional signals remains a significant challenge. Prevailing methods typically aggregate signals using fixed weights, ignoring non-stationary learning dynamics. Consequently, static strategies risk over-optimizing dimensions with diminishing returns while neglecting bottlenecks \cite{drbo,yhc,shen2025promptingenoughexploringknowledge}. This issue is further exacerbated by data heterogeneity, where a training example that is highly informative for one criterion (e.g., correctness) may be suboptimal for another (e.g., fluency). As a result, static paradigms lack the flexibility to adapt to evolving bottlenecks and varying data utility.


To address these challenges, methods such as RaR \cite{rar} and MPO \cite{mpo} implicitly synthesize multi-criterion into a single reward signal by incorporating multiple dimensions into a prompt for judge models to output a holistic score, which obscures the granularity of supervision and hinders the model from localizing specific optimization directions. Alternatively, dynamic strategies like DRBO \cite{drbo} and MDO \cite{mdo} attempt to mitigate weaknesses by prioritizing objectives with lower scores. However, these approaches overlook data heterogeneity and risk leveraging inappropriate data samples for the targeted capabilities, leading to inefficient optimization and inter-objective interference. Conversely, Omni-Thinker \cite{ominithinker} and Rubicon \cite{rubric-anchors} address data variance through curriculum-style schedules that transition from strongly constrained tasks to weakly constrained generation, yet rely on static progression plans that lack the flexibility to adapt to the real-time evolution of model capabilities.

To overcome these limitations, we propose \ourmethod, a framework that establishes an automated, \textbf{S}elf-\textbf{P}aced curriculum for RL \textbf{A}lignment by perceiving learning progress to synchronize \textbf{R}eward Dynamics with \textbf{D}ata utility. Specifically, we treat learning progress as a signal to dynamically adjust reward weights, directing the model's attention toward dimensions with significant remaining improvement potential. In parallel, \ourmethod implements adaptive data prioritization by upweighting sample categories that are highly aligned with stage-specific objectives and yield the largest marginal gains. This integrated mechanism ensures that as the model evolves, limited training compute remains precisely focused on the most promising objectives and the most informative data samples.

To sum up, our contributions are threefold:

\begin{itemize}
    \item We propose \ourmethod, an automated curriculum framework that leverages real-time learning progress to enable self-paced learning for complex, open-ended generation tasks, dynamically guiding capability acquisition through increasingly challenging stages.
    \item We introduce a unified optimization framework that couples reward weight adjustment with adaptive data importance weighting. This closed-loop system synchronizes learning intent with data utility, overcoming the limitations of single-sided curriculum strategies.
    \item Extensive experiments across multiple benchmarks demonstrate that SPARD consistently enhances model capabilities across diverse dimensions. Further analysis validates the framework's advantages in learning efficiency and stability, validating the effectiveness of our proposed framework.
\end{itemize}

\section{Related Works}

\paragraph{Reinforcement Learning Alignment via Feedback}

To navigate increasingly sophisticated LLM scenarios, hybrid reward strategies integrate multidimensional feedback signals to satisfy fine-grained quality benchmarks across open-ended tasks \cite{rlmr,openrubircs,spct}. For example, Writing-Zero \cite{writing-zero} uses a Pairwise Generative Reward Model to convert self-critique into verifiable feedback for creative writing. QA-LIGN \cite{QA-LIGN} and RLCF \cite{rlcf} further decompose evaluation into explicit principles, delivering fine-grained feedback that targets specific issues in logic or style. This idea has also been extended to multimodal reasoning, where process-level feedback guides step-by-step reasoning alongside outcome rewards \cite{autorubric-r1v}. Despite these advances, optimizing multiple forms of feedback remains challenging: most methods adopt static aggregation, which is often unable to adapt to shifting training dynamics, limiting performance gains.

\paragraph{Curriculum Learning for Reinforcement learning}
Curriculum learning structures training by progressing from easier to harder examples and is widely used in reinforcement learning to stabilize optimization \cite{kimik15,lightrl}. Rubicon \cite{rar} and Omni-Thinker \cite{ominithinker} follow a similar two-stage scheme, first training on strongly constrained tasks and then fine-tuning on more open-ended questions. However, these curricula are typically static and fail to adapt to the model’s evolving competence. Beyond static schedules, some methods \cite{selfevolve,dump} estimate difficulty from model-based priors and cast data reweighting as a multi-armed bandit problem to adjust sampling weights online, but they still rely on heuristic difficulty signals or annotations, limiting applicability when difficulty is ambiguous. In contrast, we propose a method that adaptively schedules both reward objectives and data importance based on online learning progress, allowing the curriculum to emerge from feedback rather than a fixed syllabus.

\section{Method}

\begin{figure*}
    \centering
    \includegraphics[width=0.95\textwidth]{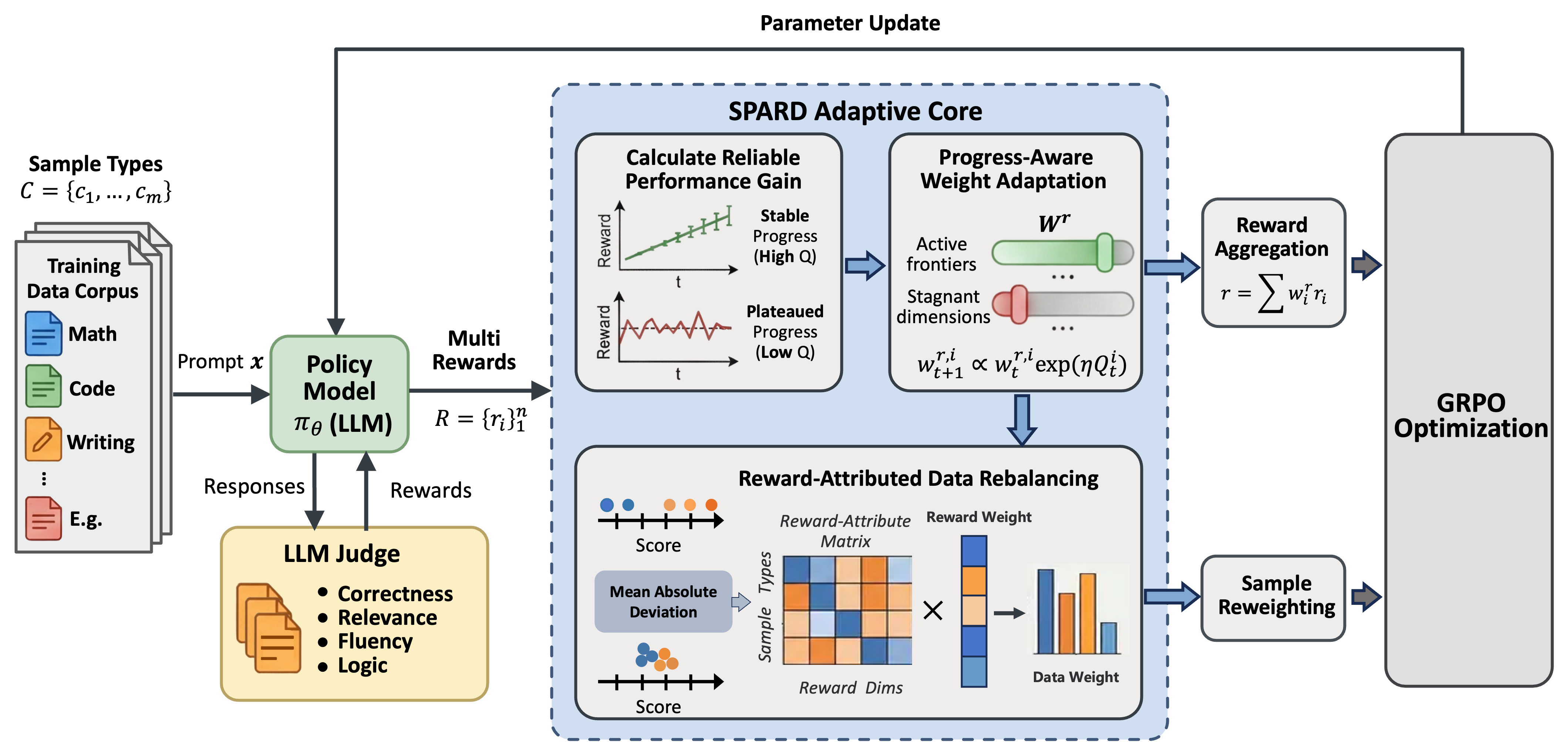}
    \caption{The framework of \ourmethod, which consists of two main synergistic mechanisms: (1) \textbf{Progress-Aware Weight Adaptation} dynamically adjusts reward weights ($\mathbf{w}^r$) based on the reliability of performance gains, and (2)\textbf{ Reward-Attributed Data Rebalancing} computes data weights ($\mathbf{w}^d$) by aggregating reward importance via a reward-attribute matrix derived from score dispersion. These components jointly guide the optimization to prioritize current learning objectives and leverage the most efficient data.}
    \label{fig:method}
    \vspace{-0.3cm}
\end{figure*}

\subsection{Preliminaries}

\paragraph{Task Formulation}
 An LLM $\pi_{\theta}$ (with parameters $\theta$) defines a probability distribution over response sequences $y$ given a query $x \sim \mathcal{D}$. To align LLMs with desired behaviors, we formulate language generation as a reinforcement learning (RL) problem.  The policy $\pi_{\theta}$ receives a scalar reward $r(x, y) \in \mathbb{R}$ that reflects the quality of the generation. The training objective is to optimize the policy parameters $\theta$ to maximize the expected reward over the dataset: 

\vspace{-0.2cm}
\begin{equation}
    J(\theta) = \mathbb{E}_{x \sim \mathcal{D}, y \sim \pi_{\theta}}[r(x, y)].
\end{equation}

\paragraph{Group Relative Policy Optimization (GRPO)}
To optimize the policy efficiently without an additional value network, we employ GRPO algorithm \cite{deepseekmath}. For each query $x$, the algorithm samples a group of $G$ outputs $\{y_i\}_{i=1}^{G}$ from the old policy $\pi_{\theta_{\text{old}}}$. The policy $\pi_\theta$ is updated by maximizing the following surrogate objective:
\vspace{-2mm}
\begin{equation}
\begin{split}
    & \mathcal{L}_{\text{GRPO}}(\theta) = \frac{1}{G} \sum_{i=1}^{G} \frac{1}{|y_{i}|} \sum_{t=1}^{|y_{i}|} \biggl\{ \min \Bigl( \rho_{i,t} \hat{A}_{i}, \\
    &\text{clip} \bigl(\rho_{i,t}, 1-\epsilon, 1+\epsilon\bigr) \hat{A}_{i} \Bigr) - \beta D_{\text{KL}}(\pi_{\theta} \| \pi_{\text{ref}}) \biggr\}.
\end{split}
\end{equation}
where $\rho_{i,t} = \frac{\pi_{\theta}(y_{i,t} \mid x, y_{i,<t})}{\pi_{\theta_{\text{old}}}(y_{i,t} \mid x, y_{i,<t})} $ is the importance ratio, $\epsilon$ is the clipping parameter, and $\beta$ controls the KL-divergence regularization. 
Crucially, GRPO estimates the baseline directly from group statistics. The advantage $\hat{A}_i$ for the $i$-th response is computed by standardizing the rewards within the group:
\begin{equation}
    \hat{A}_i = \frac{r_i - \mathrm{mean}(\{r_1, \dots, r_G\})}{\mathrm{std}(\{r_1, \dots, r_G\}) }.
    \label{advantage}
\end{equation}
Here, $r_i$ denotes the scalar reward for response $y_i$.
Consequently, the effectiveness of the optimization hinges heavily on the design and construction of this scalar signal $r_i$.

\paragraph{Multi-Reward Aggregation}

While scalar rewards suffice for tasks with objective ground truth, open-ended generation necessitates evaluating a diverse array of quality dimensions. We formalize this evaluation using a set of scoring criteria $\mathcal{P}=\{p_k\}_{k=1}^{N}$ to capture a comprehensive spectrum of model capabilities, where an $\text{LLM}_{\text{judge}}$ maps a response $y$ to a multi-dimensional reward vector $\mathbf{r}(y)$ such that $r_k(y) = \text{LLM}_{\text{judge}}(y, p_k)$. 

To facilitate RL optimization, existing approaches typically employ \textit{linear scalarization} to derive a unified learning signal:
\begin{equation}
r(x, y) = \sum_{k=1}^{N} w^r_k \cdot r_k(y), \quad \text{s.t.} \sum_{k=1}^{N} w^r_k = 1,
\end{equation}
where $\{w^r_k\}$ are the static weight hyperparameters. 

Although this simplifies optimization, it ignores the \textit{non-stationary learning dynamics} inherent in scaling reward dimensions. In practice, model capabilities exhibit \textit{asynchronous convergence}: different dimensions plateau at varying rates as training progresses. Enforcing fixed weights $\{w^r_k\}$ fails to adapt to this evolution, leading to inefficient gradient allocation and hindering the model's ability to achieve balanced proficiency across the entire objective space.

\subsection{Methodology}

In this section, we present \ourmethod, an RL framework that orchestrates an automated, self-paced curriculum. Diverging from fixed weighting schemes that overlook training dynamics, \ourmethod dynamically aligns the optimization trajectory with the model's evolving proficiency. At its core, the framework leverages Progress-Aware Weight Adaptation \ref{weight} to identify and prioritize capabilities within their prime learning phase. Concurrently, Reward-Attributed Data Rebalancing \ref{data} assigns adaptive importance weights to training samples, ensuring that the gradient updates are primarily driven by data that yields the highest marginal gains for these targeted objectives. The complete training process is presented in Algorithm \ref{alg:framework}.

\subsubsection{Progress-Aware Weight Adaptation} \label{weight}

This module focuses on the dynamic evolution of the reward weight vector $\mathbf{w^r}$ during training. We formulate this process as a dynamic resource allocation problem, where the objective is to direct the limited optimization budget toward dimensions that exhibit the highest learning potential. Static weighting schemes often fail to distinguish between \textit{stagnant dimensions} (where the model has reached a performance plateau) and \textit{active frontiers} (where capabilities are rapidly emerging). To bridge this gap, we treat the stable rate of improvement as a proxy for learnability, identifying dimensions where parameter updates yield the most significant and robust gains.

To capture these stable gains while filtering out transient noise, we draw on the Lower Confidence Bound (LCB) principle. We define the \textit{Reliable Performance Gain} $Q_t^i$ for the $i$-th reward as:
\begin{equation}
    \label{Q}
    Q_t^i = (\mu_t^i - \beta \sigma_t^i) - (\mu_{t-1}^i - \beta \sigma_{t-1}^i),
\end{equation}
where $\mu_t^i$ and $\sigma_t^i$ are the Exponential Moving Average (EMA) mean and standard deviation of the $i$-th reward component, and $\beta$ is a coefficient penalizing uncertainty. This formulation ensures that $Q_t^i$ is positive only when the mean improvement outweighs the variability, signaling robust acquisition of the corresponding capability. These reward statistics are updated as follows:
\begin{equation}
\begin{aligned}
    \label{stats}
    \mu_t^i &= \alpha \cdot r_t^i + (1-\alpha) \cdot \mu_{t-1}^i, \\
    \sigma_t^i &= \alpha \cdot \mathrm{std}_t^i + (1-\alpha) \cdot \sigma_{t-1}^i.
\end{aligned}
\end{equation}

To translate these progress signals into updated weights, we aim to maximize alignment with high-growth dimensions while preventing catastrophic forgetting or training instability caused by abrupt weight shifts. We formulate this as a KL-regularized Online Mirror Descent problem:
\begin{equation}
    \mathbf{w}_{t+1}^r = \operatorname*{arg\,max}_{\mathbf{w} \in \Delta_{n-1}} \left( \mathbf{Q}_t^\top \mathbf{w} - \frac{1}{\eta} \operatorname{KL}(\mathbf{w} \parallel \mathbf{w}_t^r) \right).
\end{equation}
The first term encourages the model to prioritize dimensions with the highest reliable gains, while the KL divergence serves as a proximal constraint to maintain a smooth optimization trajectory. The detailed derivation is provided in Appendix.~\ref{prop:optimal_weight_update}. The closed-form solution yields an exponentiated gradient update:
\begin{equation}
    \label{Wr}
    w_{t+1}^{r, i} = \frac{w_{t}^{r, i} \exp(\eta Q_t^i)}{\sum_{j=1}^{n} w_{t}^{r, j} \exp(\eta Q_t^j)},
\end{equation}
where $\eta$ is the learning rate for weight adaptation, controlling the sensitivity of the curriculum to recent progress. This mechanism naturally amplifies focus on fast-improving capabilities, synchronizing the optimization focus with the model's evolving proficiency frontier.

\begin{algorithm}[t]
\caption{\ourmethod Training Process}
\label{alg:framework}
\begin{algorithmic}[1]
\REQUIRE Policy $\pi_{\theta}$, N criteria $\{p_k\}_{k=1}^N$, $M$ data categories, interval $k$ , $\eta, \alpha, \beta, \mu$
\STATE \textbf{Init:} $\mathbf{w}^r, \mathbf{w}^d \leftarrow \text{Uniform}$; Stats $\mu, \sigma \leftarrow 0$
\FOR{step $t = 1, \dots, T$}
    \STATE Sample batch $\mathcal{B} = \bigcup_{j=1}^M \mathcal{B}_j$ from dataset
    \STATE Generate responses $\{y_i\}_{i=1}^G$ and evaluate reward vectors $\{\mathbf{r}_i\}_{i=1}^G$ using $\{p_k\}_{k=1}^N$
    \STATE Update statistics $\mu_t, \sigma_t$ based on current rewards \COMMENT{Eq.~\ref{stats}}

    \IF{$t \% k == 0$}
        \STATE Compute reliable gain $\mathbf{Q}_t$ and evolve reward weights $\mathbf{w}_{t}^r$ \COMMENT{Eq.~\ref{Q},~\ref{Wr}}
        \STATE Construct reward-data attribution distribution matrix $\tilde{F} \in \mathbb{R}^{N \times M}$ \COMMENT{Eq.~\ref{attribute},\ref{Norm}}
        \STATE Derive target data importance $\mathbf{u}$ from $\tilde{F}$ and $\mathbf{w}_{t}^r$ \COMMENT{Eq. ~\ref{U}}
        \STATE Update data weights $\mathbf{w}_{t}^d$  \COMMENT{Eq. ~\ref{Wd}}
    \ELSE
        \STATE $\mathbf{w}_{t}^r, \mathbf{w}_{t}^d \leftarrow \mathbf{w}_{t-1}^r, \mathbf{w}_{t-1}^d$
    \ENDIF

    \STATE Aggregate reward $r_i \leftarrow \sum_{k=1}^N w_{t,k}^r r_{i,k}$ for advantage $\hat{A}_i$ \COMMENT{Eq.~\ref{advantage}}
    \STATE $\theta \leftarrow \theta - \nabla_{\theta} \sum_j w_{t,j}^d \mathcal{L}_{\text{GRPO}}^{(j)}(\theta)$ \COMMENT{Eq. ~\ref{loss}}
\ENDFOR
\end{algorithmic}
\end{algorithm}

\subsubsection{Reward-Attributed Data Rebalancing} \label{data}

While the evolution of $\mathbf{w}^r_t$ determines the optimization \textit{direction}, the efficiency of this trajectory depends heavily on the underlying data utility. To synchronize data provision with the model's evolving proficiency, we propose a reward-attributed mechanism that realigns the importance of data categories based on their responsiveness to the identified growth areas. This process follows a structured pipeline consisting of reward-data attribution, weight aggregation, and loss reweighting.

We first quantify the sensitivity of each data category $C=\{c_j\}_{j=1}^m$ to different reward dimensions. Intuitively, a data category is most informative for a specific reward $i$ if the candidate responses for its prompts exhibit high score dispersion, providing a clear contrastive signal for the model to distinguish superior behaviors \cite{drtulu,dapo}. To formalize this, we construct an \emph{attribution matrix} $F \in \mathbb{R}^{n \times m}$, where $F_{ij}$ measures the utility of candidates in category $c_j$ along reward dimension $i$. For each prompt $b$ in a recent buffer $B_j$, we generate a set of $G$ candidate responses $\mathcal{G}_b$ and calculate the score separation using the mean absolute deviation (MAD):
\begin{equation}
\label{attribute}
F_{ij} = \frac{1}{|B_j|}\sum_{b\in B_j} \frac{1}{G}\sum_{x\in \mathcal{G}_b}\bigl|r_i(x)-\bar r_i^{(b)}\bigr|,
\end{equation}
where $\bar r_i^{(b)}$ denotes the group mean. Effectively, a larger $F_{ij}$ implies that category $c_j$ yields high-contrast supervision for reward $i$, while a small $F_{ij}$ suggests that reward signals are clustered for this category, leading to a weak gradient signal.

To translate these raw attribution scores into actionable importance weights, we first normalize $F$ such that each reward dimension $i$ induces a proper distribution over data categories. We apply a temperature-controlled Boltzmann normalization 
\begin{equation}
   \label{Norm}
    \tilde F_{ij} =\frac{\exp(F_{ij}/\mu) }{\sum_{k=1}^{m}\exp(F_{ik}/\mu)},
\end{equation}
where $\mu$ controls the sharpness of the mapping. We then define the \textbf{target data importance vector} $\mathbf{u} \in \mathbb{R}^m$ by aggregating these normalized attributions with the current reward importance $\mathbf{w}^r$ via a matrix product: 
\begin{equation}
    \label{U}
    u_{j}=\sum_{i=1}^{n} w_{i}^r\,\tilde F_{ij}. 
\end{equation}
Conceptually, $u_{j}$ is high when category $c_j$ is strongly attributed to reward dimensions that currently exhibit high learning potential. To ensure training stability, the global data weights $\mathbf{w}_t^d$ are updated via an EMA:
\begin{equation}
    \label{Wd}
    \mathbf{w}_t^d = \alpha \cdot \mathbf{u} + (1-\alpha) \cdot \mathbf{w}_{t-1}^d.
\end{equation}

Finally, $\mathbf{w}_t^d$ is used to reweight the training losses across categories. For a minibatch $\mathcal{B}=\bigcup_{j=1}^m \mathcal{B}_j$, the overall objective is formulated as:
\begin{equation}
    \label{loss}
    \mathcal{L}(\theta) = \sum_{j=1}^{m} w_j^d \cdot \frac{1}{|\mathcal{B}_j|}\sum_{x\in \mathcal{B}_j}\ell(x;\theta).
\end{equation}
By assigning higher weights to categories that are most conducive to the current optimization priorities, this mechanism ensures that the gradient updates are primarily driven by data samples that maximize cumulative optimization efficiency.



\section{Experiments}

\begin{table*}[!h]
\caption{Overall performance comparison on multiple benchmarks. The \textbf{bold} font indicates the best results and an \uline{underline} indicates the second-best results.}
\label{maintable}
\footnotesize
\renewcommand{\arraystretch}{1.15}
\setlength{\tabcolsep}{6pt}
\begin{tabular*}{\linewidth}{@{\extracolsep{\fill}} llcccccccc }
\toprule
\multicolumn{2}{l}{\multirow{2}{*}{\textbf{Methods}}} &
\multicolumn{3}{c}{\textbf{General Capability}} &
\multicolumn{2}{c}{\textbf{Creative Writing}} &
\multicolumn{2}{c}{\textbf{Chat}} &
\multirow{2}{*}{\textbf{AVG}} \\
\cmidrule(lr){3-5}\cmidrule(lr){6-7}\cmidrule(lr){8-9}
\multicolumn{2}{l}{} & \textbf{IFEval} & \textbf{GPQA}  & \textbf{LCB}   & \textbf{Arena-Hard} & \textbf{CW} & \textbf{MT-Bench} & \textbf{WildBench} &  \\
\midrule
\multicolumn{10}{l}{\textit{\textbf{Qwen2.5-7B-Instruct}}} \\
\midrule
 & Base                    & 70.79 & 33.84 & 39.75 & 10.70 & 48.09 & 77.93 & 41.72 & 46.12 \\
 & + SFT                   & 59.70 & 32.83 & 35.00 & 12.50 & 41.85 & 77.56 & 24.05 & 40.50 \\
 & + DPO                   & \uline{74.67} & 33.54 & 39.50 & 12.70 & 50.49 & 78.37 & 43.60 & 47.55 \\
 & + $\text{GRPO}_{\text{rm}}$      & 73.56 & 34.85 & 39.75 & 11.20 & 50.17 & 78.12 & 41.77 & 47.06 \\
 & + $\text{GRPO}_{\text{imp}}$     & 66.91 & 32.83 & \uline{40.00} & 12.40 & 45.88 & 78.62 & 42.47 & 45.59 \\
 & + $\text{GRPO}_{\text{avg}}$     & 73.75 & \uline{35.35} & \uline{40.00} & \uline{14.40} & \uline{50.89} & \uline{79.75} & \textbf{45.08} & \uline{48.46} \\
 & \textit{\textbf{+ Ours}}         & \textbf{75.78} & \textbf{38.38} & \textbf{41.75} & \textbf{15.60} & \textbf{52.49} & \textbf{81.38} & \uline{44.85} & \textbf{50.03} \\
\midrule
\multicolumn{10}{l}{\textit{\textbf{Qwen3-8B}}} \\
\midrule
 & Base                    & \uline{86.32} & 42.93 & 52.50 & 42.00 & 69.90 & 75.06 & 55.21 & 60.56 \\
 & + SFT                   & 84.73 & 33.33 & 45.25 & 32.70 & 57.50 & 71.62 & 22.09 & 49.60 \\
 & + DPO                   & 85.39 & 43.94 & 50.50 & 43.10 & \uline{73.15} & \uline{77.62} & 54.73 & 61.20 \\
  & + $\text{GRPO}_{\text{rm}}$      & 
 \uline{86.32} & 43.94 & 52.50 & 40.60 &  72.75 & 76.81 & 55.10 & 61.15 \\
 & + $\text{GRPO}_{\text{imp}}$     & 85.95 & 45.96 & 51.75 & 43.10 & 72.45 & 77.25 & \uline{55.73} & 61.74 \\
 & +  $\text{GRPO}_{\text{avg}}$    & \uline{86.32} & \uline{46.97} & \uline{52.75} & \uline{44.30} & 72.16 & 76.81 & \textbf{55.89} & \uline{62.17} \\
 & \textit{\textbf{+ Ours}}         & \textbf{88.17} & \textbf{49.49} & \textbf{54.75} & \textbf{45.90} & \textbf{73.95} & \textbf{78.31} & 55.29 & \textbf{63.69} \\
\bottomrule
\end{tabular*}
\end{table*}

\subsection{Experimental Setting}
\paragraph{Dataset}
We construct our dataset by selecting 5.4k prompts from the \textbf{WildChat-IF subset} \cite{wildchat}. It is sampled from WildChat’s conversational prompts, covering a broad range of user queries that closely reflect real-world scenarios. 
To improve optimization efficiency when training on this heterogeneous instruction collection, we annotate each prompt with a category label using an LLM-based classifier.
We partition the dataset into four different categories: \textbf{Code}, \textbf{Knowledge QA}, \textbf{Text Transformation}, and \textbf{Creative Writing}.
These category tags enable us to analyze the contribution of different data types during training.
For more detailed information on data classification, please refer to Appendix \ref{data construct}.

\paragraph{Baselines}

  To systematically evaluate the effectiveness of our proposed method, we compare it against several representative benchmarks. For direct alignment strategies, we include SFT and DPO, which utilize preferred responses and annotated preference pairs directly from the dataset. Regarding reward-based reinforcement learning methods, we evaluate the following approaches:
  \begin{itemize}
      \item $\text{GRPO}_{\text{rm}}$: A standard baseline that optimizes the policy using scalar signals derived from a learned reward model \cite{rlmt}.

       \item $\text{GRPO}_{\text{avg}}$: A baseline where the LLM judge generates individual rewards for each specific criterion independently. These partial rewards are then aggregated via static, uniform weighting to compute the final signal.

       \item $\text{GRPO}_{\text{imp}}$: An approach consolidating all criteria into a single prompt, delegating the implicit aggregation to the LLM judge to directly yield a final unified score \cite{rar}.
  \end{itemize}

\paragraph{Benchmarks}
We evaluate our models on a comprehensive suite of benchmarks spanning \textbf{General Capability}, \textbf{Creative Writing}, and \textbf{Chat}. Under General Capability, we examine fundamental reasoning and constraint adherence by employing \textbf{IFEval} \cite{ifeval} using loose prompt-level accuracy for verifiable instruction following, \textbf{LiveCodeBench (LCB)} \cite{livecodebench} for code generation, and \textbf{GPQA-Diamond} \cite{gpqa} to probe PhD-level scientific reasoning and domain-specific knowledge. For Creative Writing, we employ \textbf{CreativeWritingV3 (CW)} \cite{cwv3} and \textbf{Arena-Hard (AH)} \cite{arenahard1,arenahard2} to test the model’s generative flexibility and capacity for handling writing tasks. Finally, in the Chat domain, we focus on real-world interaction quality, adopting \textbf{WildBench (WB)} \cite{wildbench} to assess alignment with human intent and \textbf{MT-Bench (MT)} \cite{mtbench} for multiturn dialogue scenarios. Further details can be found in Appendix \ref{benchmark}.

\vspace{-0.2cm}
\paragraph{Implementation Details}

We evaluate our proposed method primarily using \textbf{Qwen2.5-7B-Instruct} \cite{qwenwt} and \textbf{Qwen3-8B} \cite{qwen3}. Notably, for Qwen3-8B, we explicitly suppress the internal reasoning process during both the training and inference stages.
We design eight individual reward metrics covering aspects such as instruction-following, correctness and fluency \cite{advancedif,spct,rm-r1} , which are then combined to form the final reward function.
For our method, we use \texttt{DeepSeek-R1} \cite{deepseekr1} as the reward model, and the full judging prompt are provided in Appendix \ref{prompt}. 
For $\text{GRPO}_\text{rm}$, we adopt \texttt{Skywork-v1-Llama-3.1-8B-v0.2} as the reward model  \cite{skywork}.
Regarding the implementation details, we adopt a learning rate of $1 \times 10^{-6}$, a prompt batch size of $32$, and a group size of $G=8$. For our proposed \ourmethod, the hyperparameters are configured as follows: $\alpha = 0.5$, $\beta = 0.1$, $\mu = 0.1$, and $\eta = 3 $. A comprehensive list of all hyperparameters is provided in Appendix \ref{parameters}.

\subsection{Overall Performance Evaluation}
\paragraph{\ourmethod  improves model performance  across all domains}

Table~\ref{maintable} presents the results across different
methods.
From the table we observe that \ourmethod achieves the highest overall average performance for different model backbones, ranking first in the majority of individual domains. This demonstrates our approach’s capacity to comprehensively bolster model capabilities while ensuring harmonious improvements across diverse domains. In contrast, standard SFT is prone to distribution shift \cite{selfaug} which can lead to a noticeable degradation of general capabilities despite minor gains in chat performance. Other baseline methods such as DPO and various GRPO implementations enhance chat and writing they sometimes struggle to improve or even maintain performance in rigorous tasks such as coding and instruction following. These results suggest that models might overfit to stylistic rewards which potentially erodes core reasoning abilities.

Notably, the results show that $\text{GRPO}_{\text{avg}}$ consistently outperforms both $\text{GRPO}_{\text{rm}}$ and $\text{GRPO}_{\text{imp}}$.
This performance gap suggests that fine-grained reward signals facilitate superior optimization outcomes. Specifically, explicit aggregation provides transparent guidance by decomposing optimization targets \cite{rlcf,openrubircs}.
In contrast, implicit approaches often conflate individual criteria within a monolithic score and consequently obscure specific learning signals. However, despite its competitive performance, $\text{GRPO}_{\text{avg}}$ remains inferior to \ourmethod in harmonizing multi-task capabilities. 
This limitation indicates that static aggregation acts as a performance bottleneck. 
Its rigid and fixed-weighting scheme lacks the sensitivity to prioritize optimization focus based on the real-time progress of different objectives during training. 
Conversely, \ourmethod adaptively re-weights optimization targets by monitoring learning dynamics across training stages, ultimately fostering a synergistic improvement across diverse capabilities.

\begin{figure}
    \centering
    \includegraphics[width=0.99\columnwidth]{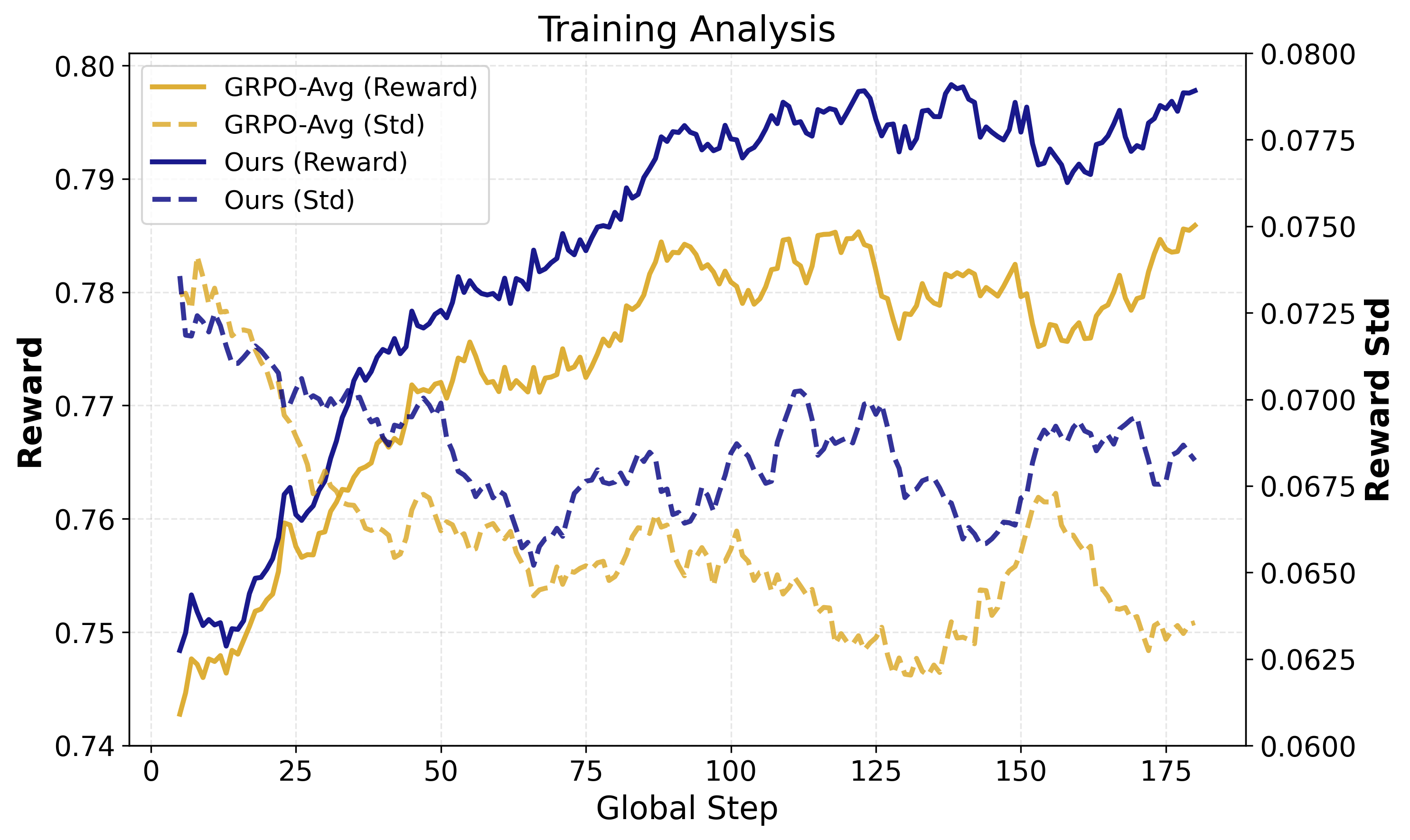}
    \caption{Training trajectories for Qwen2.5-7B-Instruct. To facilitate a clearer comparison of long-term trends and to reduce the visual impact of short-term fluctuations, all curves are smoothed using an exponential moving average (EMA).}
    \label{fig:reward}
\end{figure}

\paragraph{\ourmethod achieves faster and more stable reward improvement.}
 Figure \ref{fig:reward} illustrates the training trajectories of the \textbf{mean reward} and \textbf{standard deviation} for Qwen2.5-7B-Instruct under different training methods.  As shown in the figure, \ourmethod consistently achieves a higher average reward throughout training and exhibits a smaller variance relative to competing approaches, indicating not only stronger overall performance but also improved optimization stability and reduced sensitivity to stochasticity during training.

Detailed trajectories for each individual reward are provided in Figure~\ref {fig:mean_and_std_1} and Figure~\ref{fig:mean_and_std_2}. As illustrated in these figures, we observe pronounced gains in rewards related to creative writing and chat. 
Meanwhile, performance on other metrics remains on par with the $\text{GRPO}_\text{{avg}}$ baseline. 
This discrepancy is likely attributable to the intrinsic subjectivity and open-ended nature of these tasks, which demand imagination and creativity rather than deterministic precision. 
Such capabilities can be easily disproportionately affected when training data or reward signals impose overly rigid constraints.  
Consequently, the rigidity of static weighting makes it difficult to stimulate the full potential of the model in these domains. 
Overall, \ourmethod improves both reward maximization efficiency and training robustness, consistently delivering benefits without sacrificing the model’s broader applicability across a wide range of tasks.

\subsection{Ablation and Further Analysis}

\paragraph{Ablation Studies}
 We conduct an ablation study to validate the contributions of the key components in our framework. 
Specifically, we examined the impact of removing either Progress-Aware Weight Adaptation (PAWA) and Reward-Attributed Data Rebalancing (RADR), which constitute the core mechanisms of SPARD. The results are summarized in Table \ref{ablaton study}.
The experimental results demonstrate the effectiveness of both components. Removing PAWA leads to a significant performance drop in open-ended generation tasks, such as Creative Writing and Chat. This indicates that the dynamic weight adjustment mechanism effectively alleviates the rigid constraints imposed by static weighting strategies, thereby preserving the stylistic diversity and flexibility required for subjective tasks.
Conversely, relying solely on PAWA inhibits the improvement of general capabilities such as instruction following and coding. 
This stagnation arises because, without the reward-attribution mechanism provided by RADR, the model struggles to effectively utilize high-utility training samples (e.g., code data) that align with current optimization objectives (e.g., code generation) during the process of improving general capabilities, consequently hindering their further enhancement. Notably, the performance using either of these methods exceeds the current baseline $\text{GRPO}_{\text{avg}}$, which indicates that \ourmethod possesses strong robustness.

\begin{table}
\centering
\caption{Ablation study on the core mechanisms in \ourmethod. The \textbf{bold} font indicates the best results and an \uline{underline} indicates the second-best results.}
\label{ablaton study}
\scalebox{0.85}{
\begin{tabular}{lccccc} 
\hline
\textbf{Method} & \textbf{IF} & \textbf{GPQA} & \textbf{LCB} & \textbf{CW} & \textbf{MT}  \\ 
\hline
\multicolumn{6}{l}{\textbf{\textit{Qwen2.5-7B-Instruct}}}                                              \\ 
\hline
\ourmethod      & \textbf{75.78}       & \textbf{38.38 }        & \textbf{41.75}        & \textbf{52.49}       & \textbf{81.38}        \\
w/o PAWA        & \uline{74.86}       & \uline{37.88}        & \uline{41.75}        & 51.24       & 80.25        \\
w/o RADR        & 73.56       & 36.36         & 40.00        & \uline{51.87}       & \uline{80.93}        \\ 
\hline
\multicolumn{6}{l}{\textbf{\textit{Qwen3-8B}}}                                               \\ 
\hline
\ourmethod      & \textbf{88.17}       & \uline{49.49}         & \textbf{54.75}        & \textbf{ 73.95 }      & \textbf{78.31}        \\
w/o PAWA        & \uline{87.98}       & \textbf{51.01 }        & \uline{54.50}        &   72.41   & 77.06        \\
w/o RADR        & 86.50       & 47.47         & 52.25        &     \uline{73.15}    & \uline{77.68}        \\
\hline
\end{tabular}

}
\end{table}

\paragraph{\ourmethod is effective for different size of models}
We conduct experiments on Qwen2.5 Instruct models at multiple scales, and the results are reported in Table \ref{modelsize}. 
SPARD demonstrates strong scalability across different model sizes and consistently outperforms both the base model and the static aggregation baseline $\text{GRPO}_{\text{avg}}$ in overall evaluations. 
For smaller models (e.g., 3B and 7B), SPARD achieves broad and stable improvements. 
This suggests that multiple capabilities can benefit simultaneously during the training of smaller-scale models. 
For larger models, the gains become more selective and primarily manifest in challenging domains such as scientific reasoning and multi-turn dialogue. 
Meanwhile, performance on relatively saturated capabilities, including instruction following and code generation remains comparable to $\text{GRPO}_{\text{avg}}$. 
These results indicate that as model capacity increases, SPARD adaptively allocates optimization focus according to learning progress to maintain robust training benefits across scales.

\paragraph{Learning dynamics} 
Detailed changes in reward weights and data importance are documented in Appendix \ref{change}. 
As illustrated in Figures, reward weights undergo continuous adjustments throughout the training process to adaptively balance different objectives based on real-time feedback. 
From a data perspective, figure \ref{data change}  shows that text transformation tasks receive the highest initial weight and yield immediate gains, reflecting the rapid acquisition of instruction-following abilities. 
In contrast, the weight for code-related data peaks early and then declines, suggesting that the optimization focus shifts once the model achieves proficiency in code reasoning. 
As training progresses into the middle and late stages, knowledge QA and creative writing exhibit an upward trend in weight to occupy a larger proportion of the optimization budget. 
This pattern confirms that different capability dimensions exhibit non-stationary dynamics. 
These findings align with recent studies \cite{rar,yin2024entropylawstorydata} suggesting that verifiable tasks like coding and constrained tasks are learned earlier. 
Subjective tasks such as long-form QA and creative writing require sustained optimization due to their inherent flexibility.

\begin{table}
\centering
\caption{Model Performance Across Different Sizes. \texttt{Avg} refers to the method $\text{GRPO}_{\text{avg}}$, which averages rewards across dimensions.}
\label{modelsize}
\scalebox{0.75}{
\begin{tblr}{
  column{3} = {c},
  column{4} = {c},
  column{5} = {c},
  column{6} = {c},
  column{7} = {c},
  cell{2}{1} = {r=3}{c},
  cell{5}{1} = {r=3}{c},
  cell{8}{1} = {r=3}{},
  hline{1-2,5,8,11} = {-}{},
}
                      & \textbf{Method}         & \textbf{IF} & \textbf{GPQA} & \textbf{LCB} & \textbf{CW} & \textbf{MT} \\
\textit{\textbf{3b}}  & Base                    & 60.07            & 27.78              & 27.00           & 40.17     &    73.37      \\
                      & +Avg             & 64.51         &      30.81      & 26.75          & 43.41           &  73.68       \\
                      & \textit{\textbf{+Ours}} & \textbf{65.24}          & \textbf{31.31}            &  \textbf{28.50 }       & \textbf{44.16 }          &  \textbf{74.25 }       \\
\textit{\textbf{14b}} & Base                    &  78.03           & \textbf{46.97 }           &    46.50      & 55.80          & 79.84          \\
                      & +$\text{Avg}$             & 79.48           & 45.45             & \textbf{47.50}            & 60.02        & 82.25           \\
                      & \textit{\textbf{+ Ours}} & \textbf{80.96}           & \textbf{46.97}             & 47.25            & \textbf{60.63}           &    \textbf{84.88 }     \\
\textit{\textbf{32b}} & Base                    & 80.22       &  47.47            & 55.25              &  56.07        &   84.68         \\
                      & +Avg            & \textbf{81.70}           & 48.99             & \textbf{56.25  }          & 59.35          &  85.31       \\
                      & \textit{\textbf{+Ours}} & 80.59           &  \textbf{49.49 }           &  55.75           & \textbf{60.81 }          & \textbf{86.00 } 
\end{tblr}
}
\end{table}

\section{Conclusion}


In this work, we proposed \ourmethod, a self-paced RL framework that orchestrates a dynamic alignment curriculum.
By coupling reward dynamics with adaptive data rebalancing, SPARD resolves the inefficiencies inherent in static multi-objective optimization. Our extensive evaluation shows that SPARD consistently enhances model performance across diverse tasks while ensuring training stability. These findings underscore the necessity of progress-aware scheduling in complex alignment scenarios. For future work, we aim to generalize this framework to multimodal domains, further exploring the potential of automated curriculum learning in scaling post-training.

\section{Limitations}

While \ourmethod establishes a robust RL framework for dynamic alignment in open-ended scenes, two limitations warrant consideration. First, the framework relies on high-capability LLMs as reward judges. While this ensures alignment with complex human preferences, it introduces significant inference latency and computational overhead during the online RL loop, potentially constraining scalability and training throughput. Secondly, the current reward aggregation remains a linear approximation. This formulation may oversimplify the optimization landscape, failing to capture the intricate, nonlinear interdependencies among conflicting objectives. Future research should investigate more expressive, non-linear aggregation mechanisms to better navigate these complex relationships.

\bibliography{ref}

\newpage
\appendix

\section{Appendix}

\subsection{Proofs}
\begin{proposition}[Optimal Reward Weight Update]
\label{prop:optimal_weight_update}
Given the reliable performance gain vector $Q_t \in \mathbb{R}^n$ and the current weight distribution $w_t^r \in \Delta_{n-1}$, the closed-form solution to the regularization-constrained optimization problem defined in Eq.(7):
\begin{equation}
\label{eq:opt_problem}
w_{t+1}^r = \underset{w \in \Delta_{n-1}}{\arg\max} \left( Q_t^\top w - \frac{1}{\eta} D_{KL}(w \parallel w_t^r) \right)
\end{equation}
is given by the exponentiated gradient update rule:
\begin{equation}
w_{t+1}^{r,i} = \frac{w_{t,i}^r \exp(\eta Q_t^i)}{\sum_{j=1}^n w_{t,j}^r \exp(\eta Q_t^j)}
\end{equation}
\end{proposition}
\begin{proof}
Let $J(w)$ denote the objective function. Expanding the KL-divergence term, the objective is formulated as:
\begin{equation}
J(w) = \sum_{i=1}^n Q_t^i w_i - \frac{1}{\eta} \sum_{i=1}^n w_i \ln \left( \frac{w_i}{w_{t,i}^r} \right)
\end{equation}
The negative relative entropy term is strictly concave with respect to $w$. Consequently, the optimization problem is strictly concave over the probability simplex $\Delta_{n-1}$, guaranteeing the existence of a unique global maximum.

To derive the optimal solution, we construct the Lagrangian $\mathcal{L}(w, \lambda)$ to enforce the simplex constraint $\sum_{i=1}^n w_i = 1$. The non-negativity constraints $w_i > 0$ are implicitly satisfied by the domain of the logarithmic term (acting as a barrier function). The Lagrangian is given by:
\begin{equation}
\begin{split}
\mathcal{L}(w, \lambda) &= \sum_{i=1}^n Q_t^i w_i - \frac{1}{\eta} \sum_{i=1}^n \left( w_i \ln w_i - w_i \ln w_{t,i}^r \right) \\
&\quad + \lambda \left( \sum_{i=1}^n w_i - 1 \right)
\end{split}
\end{equation}
Where $\lambda \in \mathbb{R}$ is the Lagrange multiplier associated with the equality constraint.

Taking the partial derivative with respect to $w_i$:
\begin{equation}
\begin{split}
\frac{\partial \mathcal{L}}{\partial w_i} &= Q_t^i - \frac{1}{\eta} \left( \ln w_i + 1 - \ln w_{t,i}^r \right) + \lambda \\
&= Q_t^i - \frac{1}{\eta} \ln \left( \frac{w_i}{w_{t,i}^r} \right) - \frac{1}{\eta} + \lambda 
\end{split}
\end{equation}
Rearranging the terms to isolate $\ln w_i$, we obtain:
\begin{equation}
\begin{aligned}
\frac{1}{\eta} \ln \left( \frac{w_i}{w_{t,i}^r} \right) &= Q_t^i + \lambda - \frac{1}{\eta} \\
\ln \left( \frac{w_i}{w_{t,i}^r} \right) &= \eta Q_t^i + (\eta \lambda - 1)
\end{aligned}
\end{equation}
Exponentiating both sides yields the functional form of the optimal weights:
\begin{equation}
\label{eq:functional_form_final}
w_i = w_{t,i}^r \exp(\eta Q_t^i) \cdot \exp(\eta \lambda - 1)
\end{equation}
Let $Z = \exp(\eta \lambda - 1)$ denote the normalization constant, which is independent of the index $i$. To determine $Z$, we enforce the probability constraint $\sum_{j=1}^n w_j = 1$:
\begin{equation}
\sum_{j=1}^n w_j = Z \sum_{j=1}^n w_{t,j}^r \exp(\eta Q_t^j) = 1
\end{equation}
Solving for $Z$, we find:
\begin{equation}
Z = \frac{1}{\sum_{j=1}^n w_{t,j}^r \exp(\eta Q_t^j)}
\end{equation}
Substituting $Z$ back into Eq.~\eqref{eq:functional_form_final}, we arrive at the closed-form update rule:
\begin{equation}
w_{t+1}^{r,i} = \frac{w_{t,i}^r \exp(\eta Q_t^i)}{\sum_{j=1}^n w_{t,j}^r \exp(\eta Q_t^j)}
\end{equation}
\end{proof}

\subsection{Dataset Description}
\label{data construct}
\subsubsection{Source and Composition}
The WildChat-IF dataset utilized in this work is derived from the WildChat corpus. The original WildChat corpus comprises approximately 1 million user-chatbot conversations consisting of over 2.5 million interaction turns, collected from a publicly available service powered by GPT-3.5 and GPT-4 APIs. We filtered and categorized the samples into four distinct types: Creative Writing (CW),  Text Transformation (Text), Code Generation (Code), and Knowledge QA (QA). The final dataset comprises a total of 5,760 samples. As illustrated in Figure \ref{fig:data_dist}. 
\begin{figure}[h]
    \centering
    \includegraphics[width=0.9\linewidth]{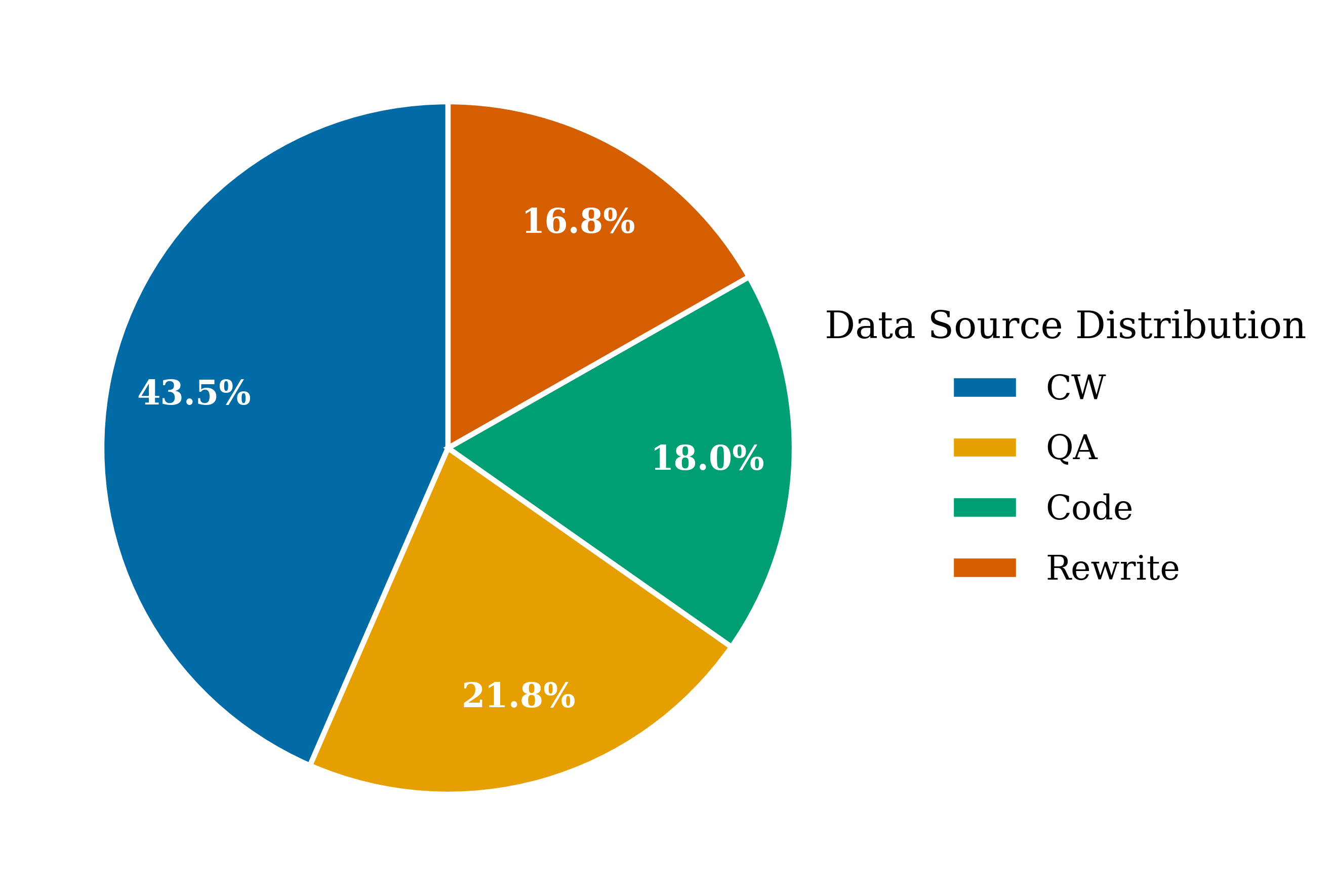}
    \caption{Distribution of selected data sources in the WildChat-IF dataset.}
    \label{fig:data_dist}
\end{figure}
\subsubsection{Data Construction} 
We employed DeepSeek-R1~\cite{deepseekr1}, a state-of-the-art language model, to process the raw data. Specifically, DeepSeek-R1 was utilized to categorize the samples based on their semantic intent. The specific prompts designed for this curation and classification process are detailed in Prompt ~\ref{prompt:user_query_classification}

\subsection{Evaluation Details}
\label{benchmark}
We evaluate SPARD using a suite of seven benchmarks organized into three distinct domains: General Capability, Creative Writing, and Chat. 

\subsubsection{General Capability}
This category targets reasoning and constraint adherence, encompassing verifiable instruction following, code generation, and domain-specific scientific knowledge. We use OpenCompass in GPQA and LCB.
\begin{itemize}
    \item Instruction-Following Evaluation (IFEval) \cite{ifeval}: IFEval assesses the objective ability of models to adhere to strict execution constraints. The dataset comprises approximately 500 prompts covering 25 types of verifiable instructions, such as word count limits and formatting requirements. Unlike model-based judges, IFEval employs programmatic metrics to calculate deterministic constraint satisfaction. We use the \textbf{prompt-level loose accuracy} metric for our result.
    \item Graduate-Level Google-Proof Q\&A (GPQA) \cite{gpqa}: GPQA contains 448 high-difficulty multiple-choice questions spanning biology, physics, and chemistry. Authored by PhD-level domain experts, these questions are designed to be "Google-proof" to resist simple retrieval. It serves as a rigorous test for expert-level reasoning and deep domain knowledge.  We use GPT-4o (version: 2024-06-01) as the judge model to calculate the \textbf{accuracy}.
    \item LiveCodeBench  (LCB) \cite{livecodebench}: LiveCodeBench evaluates code generation on contest problems published after the model's training cutoff. The benchmark measures performance via functional correctness (Pass@1) on hidden test cases, ensuring the model generalizes to novel algorithmic problems rather than recalling memorized solutions.  We report results on the \textbf{Code Generation} \textbf{ scenario}.
    
\end{itemize}
\subsubsection{Creative Writing}
This domain stresses the model's generative flexibility, evaluating its capacity to handle complex, open-ended tasks that demand stylistic nuance and high-entropy output
\begin{itemize}
    \item Creative Writing v3 (CW) \cite{cwv3}: CW consists of 32 open-ended prompts designed to elicit nuanced literary output. Responses are evaluated by a strong judge model using a criterion focused on narrative flow and emotional depth. The scoring mechanism is specifically calibrated to minimize length bias and assess subjective quality.  We use Claude-3.7 as the judge model to calculate the \textbf{judge score}.
    \item Arena-Hard V2.0(AH) \cite{arenahard1,arenahard2}: AH utilizes 500 challenging prompts curated from the Chatbot Arena, selected for their high separability. The evaluation employs a pairwise comparison mechanism where a judge model compares the target against a baseline. The resulting win-rates correlate highly (98.6\%) with human preference rankings, serving as a proxy for performance on complex queries. We use GPT-4o as the judge model to calculate the \textbf{win rate} in the \textbf{creative writing} subset and the baseline model is GPT-o1.
\end{itemize}
\subsubsection{Chat}
This category focuses on real-world interaction quality, assessing robustness against diverse, noisy user intents and the maintenance of coherence across multi-turn dialogues. 
\begin{itemize}
    \item MT-Bench (MT) \cite{mtbench}: MT-Bench assesses conversational flow and instruction following through 80 high-quality multi-turn questions across eight domains. Each task involves a two-turn dialogue to test context retention. A \textbf{judge grades} responses on a scale of 1 to 10 based on helpfulness, relevance, and accuracy. We use GPT-4o as the judge model. 
    \item WildBench (WB) \cite{wildbench}: Derived from the WildChat corpus, WildBench evaluates models on 1,024 real-world tasks that reflect diverse and noisy user interactions. It uses fine-grained, checklist-based pairwise comparisons (WB-Reward/Score) to assess practical utility across use cases like debugging and information seeking. We use GPT-4o as the judge model and use the \textbf{WB-Reward} as our metric.
\end{itemize}
\subsection{Additional Results}
\subsubsection{Training Reward Definition}
To comprehensively evaluate the quality of generated content during the training process and to provide stable learning signals, we define eight reward functions to assess different dimensions of quality, including correctness, detail, tune, logic, relevance, instruction following, and structure. Specifically, we employ the Deepseek-R1 model as our judge model, utilizing carefully designed scoring prompts to enable the model to evaluate responses across these various dimensions. The detailed prompts are provided in Appendix \ref{prompt}.

\subsubsection{Training Reward Analysis}
\label{all_rewards}
We analyze the training dynamics of Qwen2.5-7B-instruct across eight specific reward dimensions. Figure~\ref{fig:mean_and_std_1} and Figure~\ref{fig:mean_and_std_2} present the training curves, detailing the mean reward and its corresponding standard deviation (Std.) throughout the optimization process.
\begin{figure*}[h]
    \centering
    \includegraphics[width=0.9\linewidth]{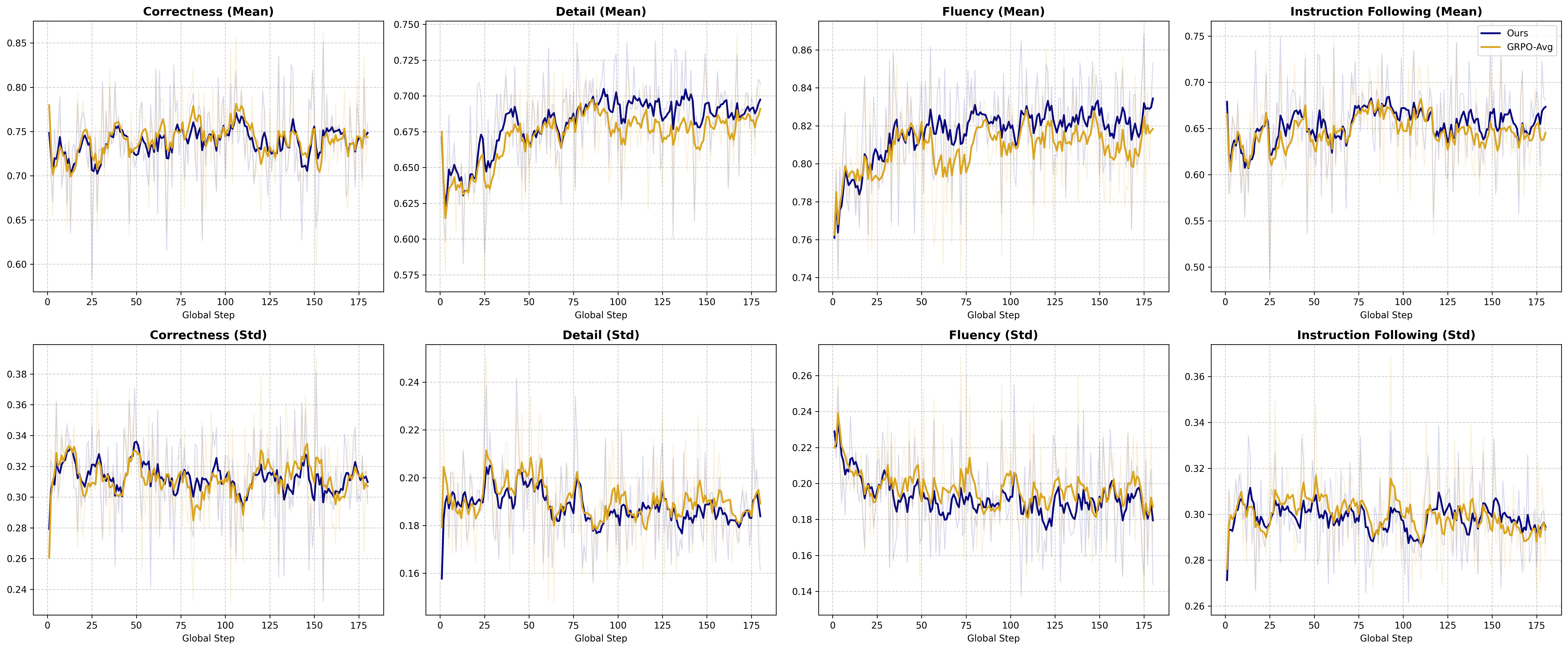}
    \caption{Training dynamics for Correctness, Detail, Fluent, Instruction Following.}
    \label{fig:mean_and_std_1}
\end{figure*}
\begin{figure*}[h]
    \centering
    \includegraphics[width=0.9\linewidth]{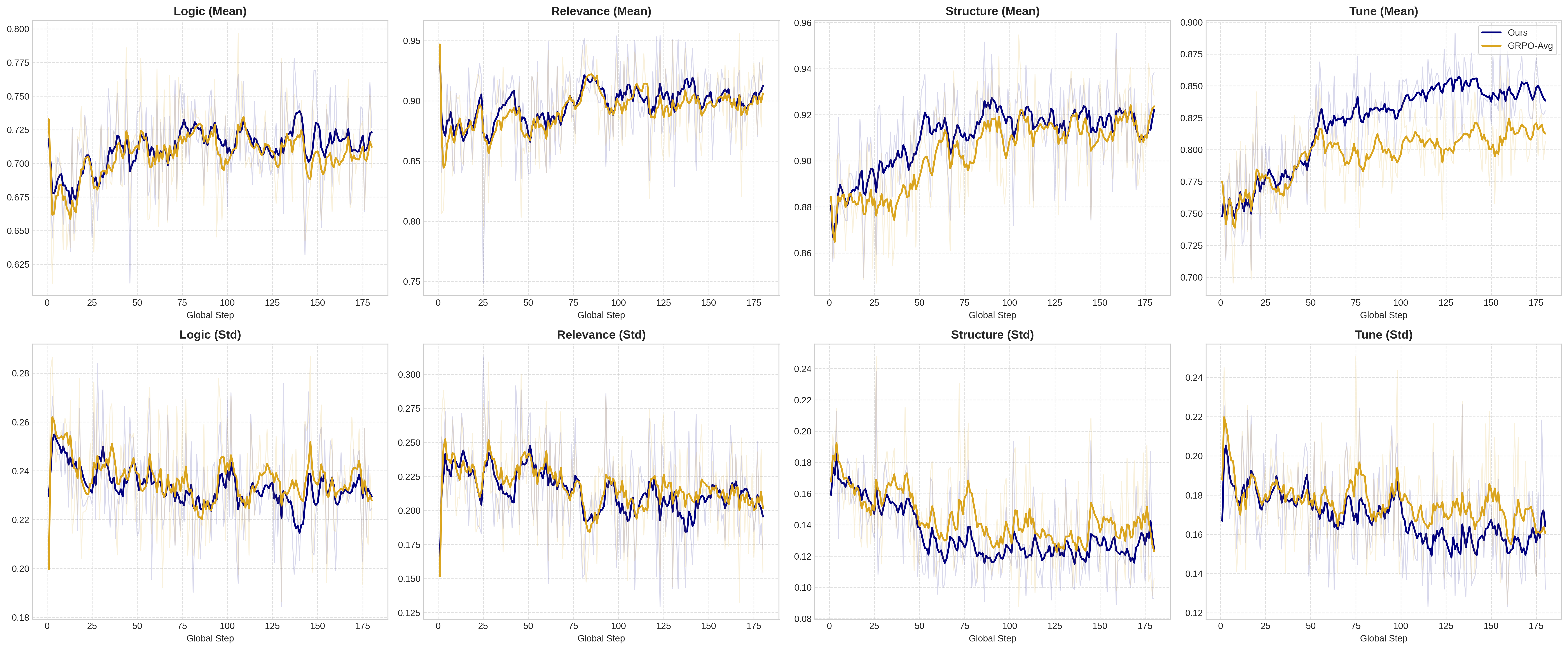}
    \caption{Training dynamics for Logic, Relevant, Structure, and Tune.}
    \label{fig:mean_and_std_2}
\end{figure*}

\par Overall, our method consistently outperforms the GRPO-Avg baseline across all evaluated metrics, demonstrating both higher reward acquisition and improved stability. Specifically: 
\begin{itemize} \item \textbf{Convergence and Performance:} As shown in Figures \ref{fig:mean_and_std_1} and \ref{fig:mean_and_std_2}, our method exhibits faster convergence rates. It achieves higher mean scores during the training process, particularly in the \textit{Structure}, \textit{Tune}, and \textit{Fluency} dimensions, suggesting a more effective alignment with the reward objectives.
\item \textbf{Training Stability:} The standard deviation plots (right columns) reveal that our method generally maintains lower or comparable variance throughout the training process. Notably, in both figures, the reduction in std implies that our policy optimization is less prone to mode collapse or instability, leading to more consistent generation quality. 
\end{itemize}

\subsubsection{Training dynamics During Training}

\label{change}
We analyzed the training dynamics during the training process, including the variation trends of reward weights and data weights. The figure \ref{data change}, \ref{reward weight change} below shows the training curves, where the changing trends of both reward weights and data weights throughout training can be observed. This demonstrates that our approach is capable of capturing the model’s continuously evolving learning progress.

\begin{figure*}
    \centering
    \includegraphics[width=0.9\linewidth]{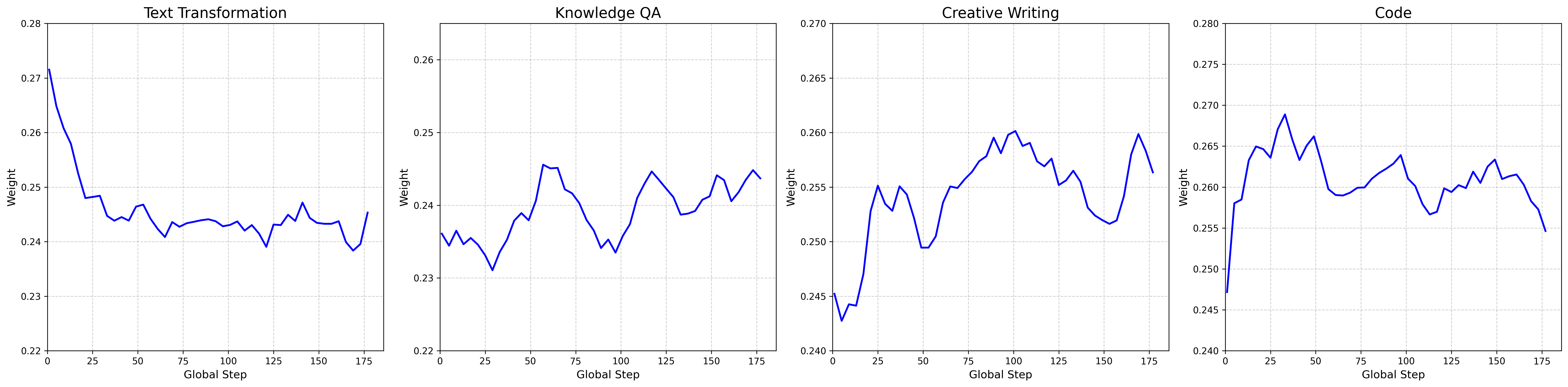}
    \caption{Data Weight Change During Training Process}
    \label{data change}
\end{figure*}

\begin{figure*}
    \centering
    \includegraphics[width=0.9\linewidth]{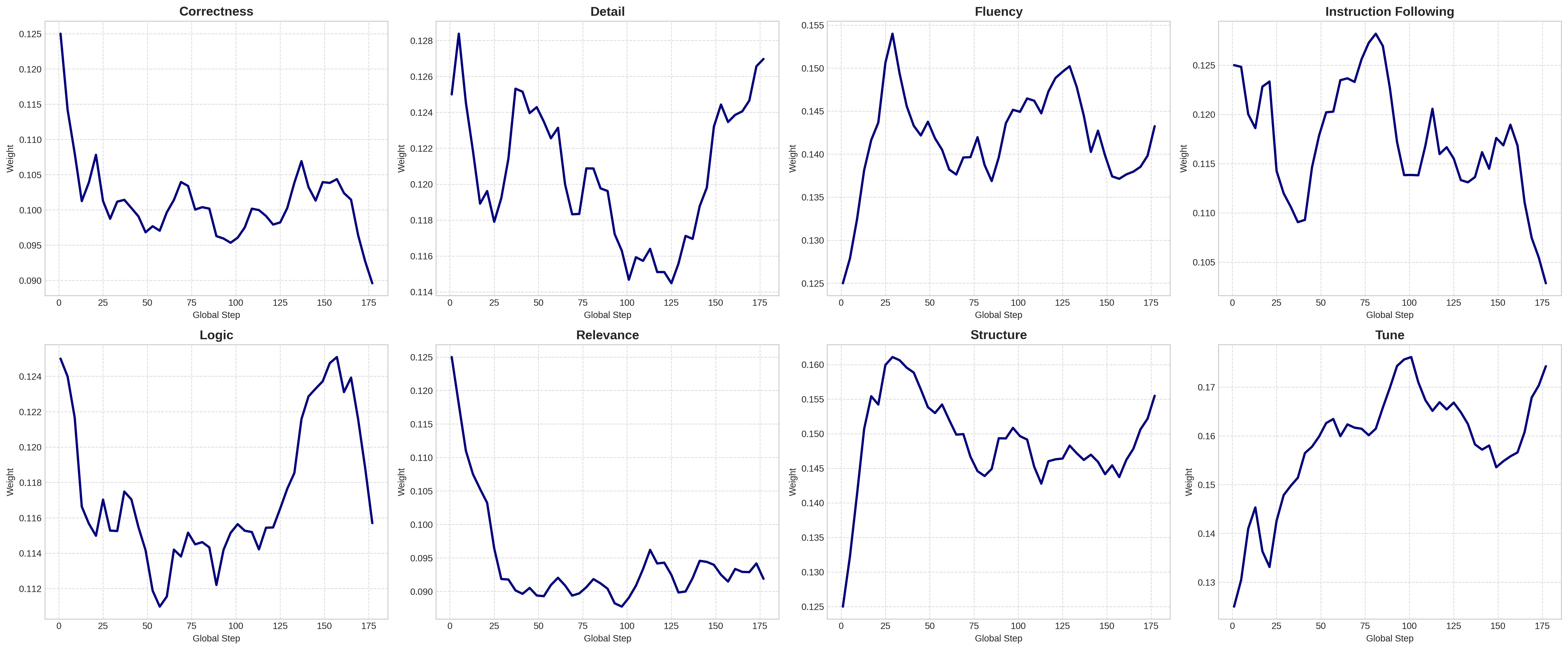}
    \caption{Reward Weight Change During Training Process}
    \label{reward weight change}
\end{figure*}

\subsubsection{Implementation Details}
\label{parameters}
We provide the corresponding hyperparameters for SFT, DPO  and GRPO in Table \ref{tab:sft}, \ref{tab:dpo}, \ref{tab:grpo}. All the training is conducted based on the ms-swift \cite{swift} framework.
\begin{table}[!h]

\centering
\caption{Supervised fine-tuning (SFT) hyperparameters used in our experiments.}
\label{tab:sft}
\begin{tabular}{lc} 
\toprule
Hyperparameter & Value  \\ 
\midrule
Batch size      & 32      \\
Epochs      & 1     \\
Learning rate   & 5e-6      \\
Warmup ratio      & 0.05       \\
Weight decay      & 0.1     \\
Adam betas      & (0.9, 0.95)      \\
LR scheduler    &    cosine    \\
\bottomrule
\end{tabular}
\end{table}
\begin{table}[!h]

\centering
\caption{Direct Preference Optimization (DPO).}
\label{tab:dpo}
\begin{tabular}{lc} 
\toprule
Hyperparameter & Value  \\ 
\midrule
Batch size      & 32      \\
Epochs      & 1     \\
Learning rate   & 1e-6      \\
DPO $\beta$  & 0.1  \\
Warmup ratio      & 0.05       \\
Weight decay      & 0.1     \\
Adam betas      & (0.9, 0.95)      \\
LR scheduler    &    cosine    \\
\bottomrule
\end{tabular}
\end{table}

\begin{table}[!h]
\centering
\caption{hyperparameters of GRPO and our method}
\label{tab:grpo}
\begin{tabular}{lc} 
\toprule
Hyperparameter & Value  \\ 
\midrule
Batch size      & 32      \\
Epochs      & 1     \\
Learning rate   & 1e-6      \\
Warmup ratio      & 0.05       \\
Weight decay      & 0.1     \\
Adam betas      & (0.9, 0.95)      \\
LR scheduler    &    cosine    \\
Group size    &   8 \\
KL coefficient      & 0.04      \\
Generation temperature      & 0.7    \\
Judge model     & Deepseek r1      \\
Judge emperature      & 0.3    \\
\ourmethod $\alpha$     & 0.5     \\
\ourmethod $\beta$     & 0.1     \\
\ourmethod $\mu$     & 0.1    \\
\ourmethod $\eta$     & 3     \\
\bottomrule
\end{tabular}
\end{table}

\subsection{Computational Cost Analysis}
\ourmethod operates within the standard GRPO framework, introducing dynamic scheduling for reward and data weights.\cite{lv2026costeercollaborativedecodingtimepersonalization} The core mechanisms, Progress-Aware Weight Adaptation (PAWA) and Reward-Attributed Data Rebalancing (RADR), rely solely on statistical aggregates (e.g., EMA and MAD) of the generated rewards. These operations are computationally negligible compared to the policy model's forward and backward passes. Unlike curriculum learning approaches that require external difficulty annotators or separate pre-sorting stages, SPARD adapts online without requiring additional inference calls. Consequently, the computational cost of SPARD remains strictly equivalent to standard Multi-Reward GRPO, while significantly improving optimization efficiency.

\subsection{Prompts}
\label{prompt}
We provide a comprehensive breakdown of the prompt protocols utilized for both data construction and evaluation to ensure full reproducibility. \cite{zhang2025killingbirdsstoneunifying,wang2025generativelargerecommendationmodels,ye2025fuxialphascalingrecommendationmodel,Wang_2025,Zhang_2025}

\par To facilitate granular analysis, Prompt~\ref{prompt:user_query_classification} stratifies user queries into four distinct taxonomies: Creative Writing, Question Answering, Code Generation, and  \cite{gu2025rapidefficientretrievalaugmentedlong}. 

\par For performance assessment, Prompt~\ref{prompt:Evaluation_Prompt} are employed to quantify response quality comprehensively($\text{GRPO}_\text{imp}$).  In this unified prompt structure, we have consolidated the detailed grading rubrics into a single variable placeholder: \texttt{\{\{all criteria\}\}}. When the model executes this prompt, it will reference the full set of injected criteria to perform a holistic evaluation and generate a comprehensive score based on the combined weights of these standards.
\par We consolidated eight detailed grading rubrics (including Correctness, Relevance, Level of Detail, Fluency, Logical Flow, Instruction Adherence, Structure, and Tone). Additionally, we provide a representative example, Prompt~\ref{prompt:level_of_detail}, alongside a unified template, Prompt~\ref{prompt:general_prompt}, for other unlisted tasks. You can replace \{\{DIMENSION\}\} (evaluation dimensions), \{\{FOCUS\_AREA\}\} (areas of focus), and \{\{LEVEL\_...\}\} (specific scoring criteria) in the following template with the actual task content.

\begin{promptbox*}[label=prompt:user_query_classification]{User Query Classification System}

\small 
\textbf{[Task Description]}\\
You are an AI assistant responsible for analyzing user queries. Your objective is to classify the user's input into one of four distinct categories based on the content and intent.

\vspace{0.5em}
\textbf{[Classification Criteria]}
\begin{itemize}
    \setlength\itemsep{0pt}
    
    \item \textbf{0. Code \& Software Engineering:}
    \begin{itemize}
        \item Writing code in specific languages (Python, Java, Kotlin, C++, JavaScript, SQL, HTML, CSS, etc.).
        \item Refactoring, compressing, or optimizing existing code.
        \item Creating programming exercises or coding problems.
        \item Explaining code logic, debugging, performance tuning, or API design.
    \end{itemize}

    \item \textbf{1. Knowledge QA \& Problem Solving:}
    \begin{itemize}
        \item STEM questions (Math, Statistics, Physics, Chemistry, Biology calculations).
        \item Exam creation/solving: MCQs, True/False, practice problems; providing answers or validating options.
        \item Analyzing, summarizing, or explaining the content of books or articles (e.g., literary genre analysis).
        \item General knowledge retrieval: listing examples, defining concepts, comparing product specifications or data.
    \end{itemize}

    \item \textbf{2. Language \& Text Rewriting:}
    \begin{itemize}
        \item Translation, paraphrasing, polishing, or grammar correction.
        \item Simplification, summarization, expansion, or shortening of text.
        \item Sentence construction based on specific grammar structures.
        \item Adjusting writing style, tone, or register.
    \end{itemize}

    \item \textbf{3. Creative Writing \& Application Scenarios:}
    \begin{itemize}
        \item Creative fiction: stories, novel chapters, character design, world-building, fan fiction, scripts, humor.
        \item Content creation: articles, speeches, movie reviews, reflections, blog posts, magazine entries.
        \item Marketing: Copywriting, headlines, slogans, product descriptions, social media posts.
        \item Generating prompts for image generation models (e.g., Midjourney).
        \item Role-play, simulated dialogues, or scenario writing.
        \item Brainstorming, ideation, and list generation requests.
    \end{itemize}
\end{itemize}

\vspace{0.5em}
\textbf{[User Query]} \\
\texttt{\{\{prompt\}\}}

\vspace{0.5em}
\textbf{[Output Format]}\\
Based on the criteria above, determine the type of the User Query and explain your reasoning. The output must be a JSON object containing the evaluation result and should not contain any other text. The \texttt{category} field must be an integer (0 for Code, 1 for Knowledge QA, 2 for Language/Text, 3 for Creative/Scenarios).

\begin{verbatim}
{
"reason": <string>,   # The reasoning behind your classification
"category": <int>     # The classification result (0, 1, 2, or 3)
}
\end{verbatim}
\end{promptbox*}

\begin{promptbox*}[label=prompt:Evaluation_Prompt]{Evaluation Prompt: Comprehensive Quality}
\small 
\textbf{[Task Description]}\\
You are an expert evaluator. Given a user query and a generated response, please rate the overall quality of the generated response on a scale of 0 to 5 based on how well the response is.

\vspace{0.5em}
\textbf{[Scoring Rules]}
\begin{itemize}
    \setlength\itemsep{0pt}
    \item Provide reasons for each score by indicating specific strengths or deficiencies within the Response. Reference exact text passages to justify the score.
    \item Be very STRICT and do not be misled by format or length.
    \item Based on the general evaluation criteria, state the weights of different criteria, and then provide an overall comprehensive score based on them.
    \item Consider all criteria holistically when determining your score.
\end{itemize}

\vspace{0.5em}
\textbf{[Criteria]}

\texttt{\{\{all criteria\}\}}

\vspace{0.5em}
\textbf{[User Query]} \\
\texttt{\{\{prompt\}\}}

\vspace{0.5em}
\textbf{[Response]} \\
\texttt{\{\{response\}\}}

\vspace{0.5em}
\textbf{[Output Format]}\\
The output should be a JSON object containing the evaluation results for the criterion, and should not contain any other text.
\begin{verbatim}
{
  "reason": <string>,  # A detailed rationale substantiating the judgment
  "score": <integer>   # The assigned score (0–5)
}
\end{verbatim}
\end{promptbox*}

\vspace{1em}
\noindent

\begin{promptbox*}[label=prompt:level_of_detail]{Evaluation Prompt: Level of Detail}
\small
\textbf{[Task Description]}\\
You are to evaluate the level of detail of the response to the user query. Focus exclusively on assessing coverage, specificity, necessary context, and actionable specifics using the criteria below.

\vspace{0.5em}
\textbf{[Criteria]}
\begin{itemize}
    \setlength\itemsep{0pt} 
    \item \textbf{Exceptional Detailed (5 points):} Complete coverage of all relevant aspects; highly specific; includes necessary context, assumptions, edge cases, and clear, actionable steps/examples.
    \item \textbf{Very Detailed (4 points):} Strong coverage with minor omissions or shallow spots; mostly specific and actionable with adequate context.
    \item \textbf{Adequately Detailed (3 points):} Covers the main aspects but leaves notable gaps; some specifics/actionable elements present, limited context or edge cases.
    \item \textbf{Partially Detailed (2 points):} Uneven coverage with significant gaps; description is general/vague in many places; lacks sufficient context or actionability.
    \item \textbf{Poorly Detailed (1 point):} Minimal coverage; mostly generic statements; little to no actionable or contextual information.
    \item \textbf{Not Detailed (0 points):} Very brief, off-topic, or lacks assessable detail.
\end{itemize}

\vspace{0.5em}
\textbf{[Scoring Rules]}
\begin{itemize}
    \setlength\itemsep{0pt}
    \item Provide reasons for each score by indicating specific strengths or deficiencies within the Response. Reference exact text passages to justify the score, ensuring that each reason is concrete and aligns with the criteria requirements while highlighting key gaps from the ideal answer.
    \item Be very STRICT and do not be misled by format or length; ensure that the Response is thoroughly evaluated beyond superficial appearances.
    \item Scoring Range: Assign an integer score between 0 to 5
\end{itemize}

\vspace{0.5em}
\textbf{[User Query]} \\
\texttt{\{\{prompt\}\}}

\vspace{0.5em}
\textbf{[Response]} \\
\texttt{\{\{response\}\}}

\vspace{0.5em}
\textbf{[Output Format]}\\
The output should be a JSON object containing the evaluation results for the criterion, and should not contain any other text.
\begin{verbatim} 
{
"reason": <string>,  # A detailed rationale substantiating the score
"score": <integer>   # An integer from 0 to 5
}
\end{verbatim}
\end{promptbox*}

\begin{promptbox*}[label=prompt:general_prompt]{User Prompt Template: Evaluation of \{\{DIMENSION\}\}}
\small
\textbf{[Task Description]}\\
You are to evaluate the \{\{DIMENSION\}\} of the response to the user query. Focus exclusively on \{\{FOCUS\_AREA\}\} using the specified criteria.

\vspace{0.5em}
\textbf{[Criteria]}
\begin{itemize}
    \setlength\itemsep{0pt}
    \item \textbf{Exceptionally Good \{\{DIMENSION\}\} (5 points):} \{\{LEVEL\_5\_DEFINITION\}\}
    \item \textbf{Very  Good\{\{DIMENSION\}\} (4 points):} \{\{LEVEL\_4\_DEFINITION\}\}
    \item \textbf{Adequately Good \{\{DIMENSION\}\} (3 points):} \{\{LEVEL\_3\_DEFINITION\}\}
    \item \textbf{Partially Good \{\{DIMENSION\}\} (2 points):} \{\{LEVEL\_2\_DEFINITION\}\}
    \item \textbf{Poorly Good\{\{DIMENSION\}\} (1 point):} \{\{LEVEL\_1\_DEFINITION\}\}
    \item \textbf{Not Good At All\{\{DIMENSION\}\} (0 points):} \{\{LEVEL\_0\_DEFINITION\}\}
\end{itemize}

\vspace{0.5em}
\textbf{[Scoring Rules]}
\begin{itemize}
    \setlength\itemsep{0pt}
    \item Provide reasons for each score by indicating specific strengths or deficiencies within the Response. Reference exact text passages to justify the score, ensuring that each reason is concrete and aligns with the criteria requirements.
    \item Be very STRICT and do not be misled by format or length.
    \item Scoring Range: Assign an integer score between 0 to 5.
\end{itemize}

\vspace{0.5em}
\textbf{[Input Data]} \\
\textbf{User Query:} \texttt{\{\{prompt\}\}} \\
\textbf{Response:} \texttt{\{\{response\}\}}

\vspace{0.5em}
\textbf{[Output Format]}\\
The output should be a JSON object containing the evaluation results for the criterion, and should not contain any other text.
\begin{verbatim}
{
"reason": <string>, # A detailed rationale substantiating
"score": <integer> # The assigned score
}
\end{verbatim}
\end{promptbox*}

\label{sec:appendix}

\end{document}